\documentclass[10pt]{article}
\usepackage[T1]{fontenc}
\usepackage{lmodern}
\usepackage[letterpaper,margin=1in]{geometry}
\usepackage{amsmath,amssymb,amsthm}
\usepackage{microtype}
\usepackage{needspace}
\usepackage{graphicx}
\usepackage{tikz}
\usetikzlibrary{arrows.meta}
\usepackage{cite}
\usepackage{xurl}
\usepackage[hidelinks]{hyperref}
\hypersetup{
	pdftitle={VC Dimension of Deterministic and Stochastic Autoregressive Linear Models},
	pdfauthor={Ilan Doron-Arad, Idan Mehalel, Elchanan Mossel}
}

\AtBeginDocument{%
	\setlength{\abovedisplayskip}{6pt plus 2pt minus 1pt}%
	\setlength{\belowdisplayskip}{6pt plus 2pt minus 1pt}%
	\setlength{\abovedisplayshortskip}{3pt plus 1pt}%
	\setlength{\belowdisplayshortskip}{3pt plus 1pt}%
}
\newtheorem{theorem}{Theorem}
\newtheorem{lemma}{Lemma}
\newcommand{\ind}[1]{\mathbf{1}\{#1\}}
\newcommand{\F}{\mathcal F}
\newcommand{\row}{\mathrm{row}}
\newcommand{\test}{\mathrm{test}}
\newcommand{\out}{\mathrm{out}}
\newcommand{\ptr}{\mathrm{ptr}}
\title{On the VC Dimension of Autoregressive Linear Threshold Functions}
\author{
	Ilan Doron-Arad\\
	{\small MIT}\\[-1pt]
	{}
	\and
	Idan Mehalel\\
	{\small The Hebrew University}\\[-1pt]
	{}\\[-1pt]
	{}
	\and
	Elchanan Mossel\\
	{\small MIT}\\[-1pt]
	{}
}
\date{}
\begin{document}
	\maketitle
	\begin{abstract}
Linear threshold functions (LTFs) are a canonical class of Boolean
predictors. In the autoregressive setting, the same LTF reads the most
recent $d$ bits, generates the next bit, and returns the bit produced
after $T$ steps. While the one-step class has VC dimension $d$, for
general $T$ the counting upper bound is $O(d^2)$. We prove that the
VC dimension is $\Theta(d^2)$ for every $T\ge d^2$, resolving the
quadratic gap for deterministic autoregressive LTFs.
We then adapt the same technique to show an exponential separation between deterministic and stochastic autoregressive models: replacing the sign of LTFs by logistic sampling can increase the VC dimension of the final majority prediction class from quadratic to exponential in \(d\). 
\end{abstract}
	
	\section{Introduction}\label{sec:introduction}
	Understanding recursive rules is central in mathematics,
	connecting number theory, probability, and dynamical systems. Known examples include the $3n+1$ problem~\cite{lagarias1985} and random Fibonacci sequences~\cite{viswanath2000}, illustrating how an elementary update rule yields notoriously difficult mathematical questions.
	
	In recent years, understanding recursive formulas has become important in
	large language models (LLMs). As an abstraction, an autoregressive model repeatedly applies
	a fixed next-token generator to the current sequence and appends the
	generated token.
	Intermediate tokens can serve as computational steps~\cite{nye2021,wei2022},
	increasing the computational power of autoregressive models~\cite{malach2024,merrill2024}.
	To study the statistical
	cost of this process, Joshi et al.~\cite{joshi2025} introduced a learning
	framework in which a fixed generator is applied repeatedly for $T$ steps and the
	last token is the answer. Training reveals either only the prompt string and the answer bit
	(\emph{end-to-end supervision (e2e)}), or reveals all intermediate
	tokens (\emph{chain-of-thought supervision (CoT)}). A fundamental question is to understand the statistical complexity of CoT and e2e supervision. 
	
	Perhaps the most fundamental test case for autoregressive learning from the work of Joshi et al.~\cite{joshi2025} onward has been the class of \emph{linear threshold
		functions} (LTFs). In the studied learning problem, an LTF outputs the sign of a weighted
	sum over the Boolean hypercube. In our autoregressive model, the same LTF
	reads the most recent $d$ bits at every step. Its weights and threshold
	never change during generation. Formally, for integers $d,T\ge1$ and $w\in\mathbb Q^d$, define a {\em autoregressive linear threshold function} (LTF)
	$f^T_w:\{0,1\}^d\to\{0,1\}$ of dimension $d$ with $T$ autoregressive steps as follows. Given
	$x=(x_{1-d},\ldots,x_0)$, set $z_t=x_t$ for $1-d\le t\le0$ and
	\begin{equation}\label{eq:model}
		z_t=\mathbf{1}\left\{\sum_{\ell=1}^d w_\ell z_{t-\ell}\ge1\right\}
		\quad(1\le t\le T). 
	\end{equation}
	Here $\ind{A}$ is the indicator function, equal to $1$ when the Boolean expression $A$ is true and $0$ otherwise. Let $f^T_{w}(x)=z_T$ and let
	$\F_{d,T}=\{f^T_{w}:w\in\mathbb Q^d\}$ be the class of all autoregressive LTFs with rational weights. Standard LTFs are obtained from this class by setting $T=1$. 

	For this class, a large gap remained in the
	VC dimension, which measures how many inputs can receive arbitrary
	labels and controls sample complexity in distribution-free
	learning~\cite{blumer1989}.
	Formally, a set $A\subseteq\{0,1\}^d$ is \emph{shattered} by
	$\F_{d,T}$ if, for every map $\lambda:A\to\{0,1\}$,
	there exists $f\in\F_{d,T}$ with $f(x)=\lambda(x)$ for all $x\in A$.
	The VC dimension is the maximum size of a set shattered by $\F_{d,T}$.
	For $T=1$, it is well known that $\operatorname{VCdim}(\F_{d,1})=d$ from~\cite{cover1965},
	even though the number of LTFs over the boolean cube is
	$2^{\Theta(d^2)}$~\cite{zuev1989}, and the VC dimension of a class may grow logarithmically with its size. 
	Since every LTF uniquely determines its autoregressive
	trajectory, the counting bound
	$\operatorname{VCdim}(\F_{d,T}) = O(d^2)$
	holds for every $T$. As for lower bounds, \cite{doron2026online} proved that $\operatorname{VCdim}(\F_{d,T}) = \Omega(d)$ for sufficiently large $T$. However,
    whether the VC dimension is of order $d$, $d^2$, or anything in between remained open, and was asked explicitly in~\cite{joshi2025,hanneke2026,doron2026online}.
	Thus, it remained unclear how autoregressive steps affect the VC dimension of LTFs which directly implies the sample complexity with e2e supervision.
	In contrast, it is known that with CoT supervision the sample complexity depends linearly on $d$. The upper bound was proved in~\cite{hanneke2026, li2026}, and the lower bound in~\cite{doron2026online}.

	\paragraph{Our results.}
	
	Our first main result is that, for $T\ge d^2$, the VC dimension of autoregressive LTFs is $\Omega(d^2)$, substantially larger than the $\Theta(d)$ VC dimension of the base (non-autoregressive) LTF class.
	
	\begin{theorem}\label{thm:main}
		For all integers $d,T\ge1$ with $T\ge d^2$,
		$\operatorname{VCdim}(\F_{d,T})\ge d^2/3024^2$.
	\end{theorem}
	
	The weights used in the proof are rational numbers of polynomial bit length.
	Thus, a single LTF,
	reused without changing its parameters, can increase VC dimension
	from linear to quadratic. In particular, at fixed accuracy
	and confidence, e2e learning requires $\Theta(d^2)$ examples,
	whereas the
    analogue CoT-supervision bound is $\Theta(d)$.
    This resolves the question of
	statistical separation left open in~\cite{joshi2025}. 

\paragraph{Implication for stochastic autoregressive learning.}
Our second result concerns stochastic autoregressive learning, replacing deterministic generation by a distribution of the next token at each step. We consider this as the canonical study case for the stochastic setting~\cite{doron2026stochastic}. 
 Let $\sigma(u)=\frac1{1+e^{-u}}$ be the sigmoid (or logistic) function. Here, we replace the
threshold in \eqref{eq:model} by an affine logistic probability:
for $w\in\mathbb R^d$, $b\in\mathbb R$, and a prompt
$x=(x_{1-d},\ldots,x_0)$, initialize $Z_t=x_t$ for $1-d\le t\le0$ and sample the next token $Z_t$ by the following distribution: 
\begin{equation}\label{eq:stochastic-model}
 \Pr(Z_t=1\mid Z_{1-d},\ldots,Z_{t-1})
 =\sigma\left(b+\sum_{\ell=1}^d w_\ell Z_{t-\ell}\right),
 \qquad \quad(1\le t\le T).
\end{equation}
Define the final
probability by $q^{(T)}_{w,b}(x)=\Pr(Z_T=1\mid x)$ and the majority class by
\begin{equation}\label{eq:stochastic-classes}
\begin{aligned}
 \mathcal H_{d,T}
 &=\left\{x\mapsto \mathbf{1}\left\{q^{(T)}_{w,b}(x)>1/2\right\}:
       w\in\mathbb R^d,\ b\in\mathbb R\right\}.
\end{aligned}
\end{equation}
Thus, $\mathcal H_{d,T}$ consists of the deterministic majority
predictions of the stochastic generators. We prove that even in this simple stochastic setting shattering exponential many points is possible, though with small margin. 

\begin{theorem}\label{thm:stochastic}
There are constants $c,C>0$ and $d_0$ such that for every $d\ge d_0$
there is an integer $T\le2^{2^{Cd}}$ for which $\operatorname{VCdim}(\mathcal H_{d,T})\ge2^{cd}.$ 
\end{theorem}
Recall that the \emph{pseudo-dimension} of a real valued function
class $\mathcal Q$ is the largest $n$ for which there exist inputs
$x_1,\ldots,x_n$ and fixed thresholds $r_1,\ldots,r_n$ such that,
for every labeling $\lambda\in\{0,1\}^n$, some $q\in\mathcal Q$
satisfies $\mathbf{1}\{q(x_i)>r_i\}=\lambda_i$ for all $i\in[n]$.
In particular, Theorem~\ref{thm:stochastic} shows that the class
$\{q^{(T)}_{w,b}:w\in\mathbb R^d,\ b\in\mathbb R\}$
has pseudo-dimension $2^{\Omega(d)}$ for some $T$, even with
all thresholds fixed at $r_i=1/2$, in stark contrast to the
$O(d^2)$ bound in the deterministic setting. Previous work show qualitative differences between deterministic
and stochastic autoregressive learning~\cite{doron2026online,doron2026stochastic},
but the separating constructions relied on specially designed next-token
distributions. In contrast, our exponential separation holds even for
the canonical logistic model, where each next token probability is a
monotone function of an affine linear score. 


	\paragraph{Further related work.}
	Joshi et al.~\cite{joshi2025} also establish a computational
	separation between e2e and CoT supervision: autoregressive LTFs are efficiently learnable with CoT
	supervision, whereas e2e learning is hard under cryptographic
	assumptions. Magen, Joshi, and Vardi~\cite{magen2026} show that a
	single autoregressive threshold neuron can simulate threshold
	circuits with polynomial overhead, after a fixed encoding of the
	inputs. Joshi et al.~\cite{joshi2026thinkers} study learning from
	multiple thinkers and show that CoT supervision from even a small
	number of thinkers can make learning computationally hard.
	The autoregressive learning framework has also been extended to
	online prediction~\cite{doron2026online} and stochastic next-token
	generation~\cite{doron2026stochastic}. Li~\cite{li2026} obtains optimal realizable PAC
	bounds for learning complete CoT trajectories over multiclass
	alphabets, independently of trajectory length.

	\subsection{Overview of the techniques}\label{sec:overview}
	\subsubsection*{Theorem~\ref{thm:main}: deterministic generators}
	Our goal is to shatter $\Theta(d^2)$ points by a class of autoregressive LTFs. We arrange the labels we wish to realize by this class as an $m\times m$ binary matrix
	$\lambda$, where $m=\Theta(d)$. Each input will specify one row $i$ and
	one entry $s$ in that row. We encode row $i$ as the integer
	\[
	\theta_i=\sum_{s=1}^m2^{s-1}\lambda_{i,s}.
	\]
	Thus, $m$ numbers, one number for each row, contain the information on all $m^2$ labels. We place these numbers in
	selected weights and construct fixed inputs whose final outputs are
	the requested binary digits. Encoding labels in weights and extracting
	their digits is a familiar approach in neural network VC lower
	bounds~\cite{bartlett2019}. Here, however, the challenge is to perform the entire extraction with a single fixed LTF, while storing the current state of the recursion in only $O(m)=O(d)$ previously generated bits (the {\em memory}). We accomplish this by first constructing a recursive formula shattering sufficiently many points with several LTFs per round. Then, we explain how to reduce this formula to a single LTF, and finally, explain how to choose the input points independent of the labels. 
	
	\paragraph{Shattering with several LTFs.}
	We recover the entries of a row using several LTFs, allowing first $O(m)$ LTF computations in each autoregressive round $t$. The
	recursion keeps four groups of bits: {\em row} bits used to select the row $i$ corresponding to the number $\theta_i$;
	One {\em test} bit computing a linear threshold test;
	{\em output} bits storing the digits already recovered, and
	{\em pointer} bits selecting the next digit to be discovered. Starting with zero output,
	we recover the digits from largest to smallest. To recover digit
	$s$, the test checks whether $\theta_i$ minus the value represented
	by the current output is at least $2^{s-1}$. Its answer is exactly the label
	$\lambda_{i,s}$. The corresponding output bit is then updated and
	the pointer moves to the next digit; the other output bits and the
	selected row remain unchanged. Every operation is an LTF computation,
	including the conditional update of an output bit. After $O(m)$
	rounds, all entries of the selected row have been recovered and remain
	stored.
	
	\paragraph{Implementing each round with one LTF.}
	We encode the entire state of round $t$ of the above recursion in a {\em block} $E^t$ of $O(m)$
	bits. Each block contains a fixed initial $1$ marking the beginning of a block, designated positions
	for the four groups of bits (row, test, output, pointer), and zeros elsewhere. The same LTF
	generates the next block $E^{t+1}$ based on the current block $E^t$, one bit at a time. Within each new block $E^{t+1}$, the test bit is generated before the output bits, allowing the output bits to depend on both the new test bit and the previous block $E^t$. To implement all updates with the same LTF, we carefully choose the {\em positions} of the bits within each block. Using modular arithmetic, we ensure that coefficients intended for one update do not interfere with others. Each test, output, and pointer bit is stored in two consecutive positions. The LTF computes the bit at the first position, while at the second bit, pairs of opposite weights cancel the contributions of earlier bits, allowing the LTF to {\em copy} the bit just generated. Consequently, one fixed LTF transforms $E^t$ into $E^{t+1}$ using $O(m)$ generated bits per round. The full procedure requires $\Theta(m)$ blocks and therefore takes $\Theta(m^2) = \Theta(d^2)$ autoregressive steps.
	
	\paragraph{Choosing the inputs.}
	Recovering a row is not yet enough: all $m^2$ inputs must be fixed
	before the labels are chosen, and every requested entry must be
	returned at the same time $T$. For each row, concatenate enough
	initial zeros with its first two blocks $E^0,E^1$, both of which
	are independent of the labels. Choose the input for entry $s$ to be a window
	of $d$ consecutive bits in this string, shifted by an amount $r_s$.
	We choose the shift $r_s$ so that time $T$ falls at the output position
	storing entry $s$ of a later block. Once the required digit is extracted,
	the output bits remain fixed across blocks, so this works for every
	sufficiently large $T$. The resulting $m^2$ inputs
	realize the arbitrary matrix $\lambda$ under one common weight
	vector, proving the quadratic lower bound. 
	As a sanity check, considering the $O(d^2)$ upper bound, note that our construction indeed cannot shatter more than $O(d^2)$ inputs, since selecting among $m$ rows and storing their $m$ recovered digits requires $\Theta(m)$ bits in the $d$-bit history, limiting $m$ to $O(d)$.

\subsubsection*{Theorem~\ref{thm:stochastic}: stochastic generators}

Unlike the deterministic construction, we encode exponentially many
labels in a single real number $\alpha\in(0,1)$, using its base-$4$
digits. Each $m$-bit input selects one of $N=2^m$ digits.
To recover the required bit, we repeatedly sample a Bernoulli bit whose odds
are proportional to $4^j\alpha$, where $j$ is the selected digit's index.
These odds are implemented by a logistic gate with an affine score.
The key observation is that, although the sampling probability
is monotone, the probability that the number of sampled ones
is $0$ or $1$ modulo $4$ oscillates around $1/2$.
With sufficiently many samples, the sign of this difference
reveals the selected digit, and hence its label.
We implement this experiment with a single logistic generator
by adapting the deterministic block construction, using
$O(m)$ memory to shatter $2^m$ inputs.

	\paragraph{Organization.}
	Section~\ref{sec:several} gives the recursion using
	several LTFs, Section~\ref{sec:single} implements it with one LTF,
	and Section~\ref{sec:inputs} chooses the inputs and proves
	Theorem~\ref{thm:main}. Sections~\ref{sec:stochastic-digits}--\ref{sec:stochastic-inputs}
	give the stochastic digit reader, implement it with one logistic
	generator, and prove Theorem~\ref{thm:stochastic}.

	\section{Shattering with Several LTFs}\label{sec:several}
	
	In this section, we shatter $\Theta(m^2) = \Theta(d^2)$ inputs using $O(m)$ LTFs updated in each round $t$.  Then, in the next section we show how to reduce them to a single more involved LTF. 
	Fix $m\ge2$ and an arbitrary labeling $\lambda\in\{0,1\}^{m\times m}$.
	Write $[m]=\{1,\ldots,m\}$. Define the {\em encoding of row} $i$ by
	\begin{equation*}
		\theta_i=\sum_{j=1}^m2^{j-1}\lambda_{i,j}
		\qquad(i\in[m]).
	\end{equation*}
	Let $\theta=(\theta_1,\ldots,\theta_m)$ be the {\em encoding} of $\lambda$. We will compute the
	binary digits of $\theta_i$ from the largest to the smallest using a
	recursive formula.
	For the moment, we allow several simple LTF computations in each round $t$.
	Specifically, for each selected row $i\in[m]$ and every $t \ge 0$, we define $3m+2$ bits $f^t_{\theta,i}\in\{0,1\}^{3m+2}$ interpreted as follows, by the displayed order.  
	
	\begin{itemize}
		\item $f^t_{\theta,i,\row,j}$ for $j\in[m]$. These row bits select the encoded row. 
		\item $f^t_{\theta,i,\test}$. This
		test bit compares the encoded row with the current output and pointer values.
		\item $f^t_{\theta,i,\out,j}$ for $j\in[m]$. These output bits store the recovered digits of the selected row. 
		\item $f^t_{\theta,i,\ptr,j}$ for $j\in[m+1]$. These pointer bits will specify the next digit. 
	\end{itemize}
	
	We now define $f^t_{\theta,i}\in\{0,1\}^{3m+2}$ formally. For each $i\in[m]$, initialize the values at $t=0$ by
	\begin{equation}\label{eq:init}
		\begin{aligned}
			f^0_{\theta,i,\row,j}&=\ind{j=i},&
			f^0_{\theta,i,\test}&=f^0_{\theta,i,\ptr,m+1}=1, &
			f^0_{\theta,i,\out,j}&=f^0_{\theta,i,\ptr,j}=0
			\quad(j\in[m]).
		\end{aligned}
	\end{equation}
	For each $i\in[m]$ and $t\ge0$, define $f^{t+1}_{\theta,i}$ from $f^t_{\theta,i}$ by
	\begin{align}
		&f^{t+1}_{\theta,i,\row,j}
		=f^t_{\theta,i,\row,j} &&(j\in[m]),\label{eq:row}\\
		&f^{t+1}_{\theta,i,\test}
		=\mathbf{1}\left\{\sum_{j=1}^m\theta_jf^t_{\theta,i,\row,j}
		-\sum_{j=1}^m2^{j-1}
		\bigl(f^t_{\theta,i,\out,j}+f^t_{\theta,i,\ptr,j}\bigr)\ge0\right\},
		\label{eq:test}\\
		&f^{t+1}_{\theta,i,\out,j}
		=\mathbf{1}\left\{2f^t_{\theta,i,\out,j}+f^t_{\theta,i,\ptr,j}
		+f^{t+1}_{\theta,i,\test}\ge2\right\}
		&&(j\in[m]),\label{eq:out}\\
		&f^{t+1}_{\theta,i,\ptr,j}
		=f^t_{\theta,i,\ptr,j+1} &&(j\in[m]),\label{eq:ptr}\\
		&f^{t+1}_{\theta,i,\ptr,m+1}=0.\label{eq:last}
	\end{align}
	The test bit $f^{t+1}_{\theta,i,\test}$ is evaluated before the output bits $f^{t+1}_{\theta,i,\out,j}$; the output bits can
	then be evaluated in parallel. Note that the above equations can all be implemented as LTFs. In particular, copying a bit and producing the
	constant zero are LTF operations. Thus, these equations use only
	LTF computations. An illustration is given in Figure~\ref{fig:recursion-transition}. 
	The next lemma shows that this recursion recovers every entry of the arbitrary labeling $\lambda$. 
	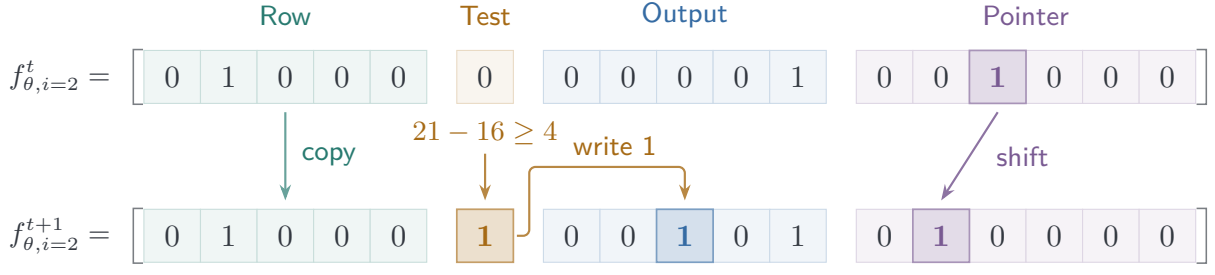
\begin{figure}[t]
		\centering
		\begingroup
		\definecolor{transitionrow}{HTML}{287A70}
		\definecolor{transitiontest}{HTML}{A66A15}
		\definecolor{transitionout}{HTML}{356BA3}
		\definecolor{transitionptr}{HTML}{795A92}
		\definecolor{transitionink}{HTML}{333840}
		\resizebox{.96\linewidth}{!}{%
			\begin{tikzpicture}[
				x=1cm,y=1cm,
				font=\small,
				text=transitionink,
				every node/.style={inner sep=0pt},
				arrow/.style={-{Stealth[length=1.8mm,width=1.3mm]},line width=.7pt}
				]
				\foreach \start/\count/\c in {
					0/5/transitionrow,
					3.54/1/transitiontest,
					4.52/5/transitionout,
					8.06/6/transitionptr}{
					\foreach \y in {0,-1.80}{
						\foreach \j in {1,...,\count}{
							\pgfmathsetmacro{\leftedge}{\start+.64*(\j-1)}
							\filldraw[fill=\c!7,draw=\c!25,line width=.35pt]
							(\leftedge,\y-.30) rectangle ++(.64,.60);
						}
					}
				}
				\filldraw[fill=transitiontest!21,draw=transitiontest!65,line width=.55pt]
				(3.54,-2.10) rectangle (4.18,-1.50);
				\filldraw[fill=transitionout!21,draw=transitionout!65,line width=.55pt]
				(5.80,-2.10) rectangle (6.44,-1.50);
				\filldraw[fill=transitionptr!21,draw=transitionptr!65,line width=.55pt]
				(9.34,-.30) rectangle (9.98,.30);
				\filldraw[fill=transitionptr!21,draw=transitionptr!65,line width=.55pt]
				(8.70,-2.10) rectangle (9.34,-1.50);
				
				\node[text=transitionrow] at (1.60,.70) {\textsf{Row}};
				\node[text=transitiontest] at (3.86,.70) {\textsf{Test}};
				\node[text=transitionout] at (6.12,.70) {\textsf{Output}};
				\node[text=transitionptr] at (9.98,.70) {\textsf{Pointer}};
				
				\node[anchor=east] at (-.28,0) {$f^t_{\theta,i=2}={}$};
				\node[anchor=east] at (-.28,-1.80) {$f^{t+1}_{\theta,i=2}={}$};
				\foreach \j/\b in {1/0,2/1,3/0,4/0,5/0}{
					\node[font=\normalsize] at ({.64*(\j-.5)},0) {$\b$};
					\node[font=\normalsize] at ({.64*(\j-.5)},-1.80) {$\b$};
				}
				\node[font=\normalsize] at (3.86,0) {$0$};
				\node[font=\normalsize,text=transitiontest] at (3.86,-1.80) {$\mathbf{1}$};
				\foreach \j/\b in {1/0,2/0,3/0,4/0,5/1}{
					\node[font=\normalsize] at ({4.52+.64*(\j-.5)},0) {$\b$};
				}
				\foreach \j/\b in {1/0,2/0,4/0,5/1}{
					\node[font=\normalsize] at ({4.52+.64*(\j-.5)},-1.80) {$\b$};
				}
				\node[font=\normalsize,text=transitionout] at (6.12,-1.80) {$\mathbf{1}$};
				\foreach \j/\b in {1/0,2/0,4/0,5/0,6/0}{
					\node[font=\normalsize] at ({8.06+.64*(\j-.5)},0) {$\b$};
				}
				\node[font=\normalsize,text=transitionptr] at (9.66,0) {$\mathbf{1}$};
				\foreach \j/\b in {1/0,3/0,4/0,5/0,6/0}{
					\node[font=\normalsize] at ({8.06+.64*(\j-.5)},-1.80) {$\b$};
				}
				\node[font=\normalsize,text=transitionptr] at (9.02,-1.80) {$\mathbf{1}$};
				
				\foreach \y in {0,-1.80}{
					\draw[transitionink!65,line width=.5pt]
					(-.01,\y+.31) -- (-.12,\y+.31) --
					(-.12,\y-.31) -- (-.01,\y-.31);
					\draw[transitionink!65,line width=.5pt]
					(11.91,\y+.31) -- (12.02,\y+.31) --
					(12.02,\y-.31) -- (11.91,\y-.31);
				}
				
				\draw[arrow,draw=transitionrow!70] (1.60,-.40) -- (1.60,-1.40);
				\node[anchor=west,text=transitionrow] at (1.78,-.90) {\textsf{copy}};
				
				\node[text=transitiontest] at (3.86,-.64) {$21-16\ge4$};
				\draw[arrow,draw=transitiontest!80] (3.86,-.87) -- (3.86,-1.40);
				
				\draw[arrow,draw=transitiontest!80,rounded corners=2pt]
				(4.22,-1.80) -- (4.36,-1.80) -- (4.36,-1.02) --
				(6.12,-1.02) -- (6.12,-1.40);
				\node[text=transitiontest] at (5.33,-.78) {\textsf{write 1}};
				
				\draw[arrow,draw=transitionptr!85] (9.66,-.40) -- (9.02,-1.40);
				\node[anchor=west,text=transitionptr] at (9.64,-.90) {\textsf{shift}};
			\end{tikzpicture}%
		}
		\endgroup
		\caption{One round of the recursion for $m=5$, selected row $i=2$,
			and encoding $\theta_2=21$, from $t=3$ to $t+1=4$.
			Within each region (Row, Test, Output, Pointer), coordinates are ordered by increasing $j$.
			The row bits remain unchanged. The current output bits encode
			$16$, and the pointer selects position $3$, whose weight is
			$2^{3-1}=4$. Since $21-16\ge4$, the new test bit is $1$,
			so the third output bit is set to $1$, while all other output
			bits remain unchanged. Finally, the pointer shifts from
			position $3$ to position $2$, selecting the next digit to test.}
		\label{fig:recursion-transition}
	\end{figure}
	
	\Needspace{5\baselineskip}
	\begin{lemma}\label{lem:digits}
		For every labeling $\lambda\in\{0,1\}^{m\times m}$ with encoding
		$\theta$, we have $f^t_{\theta,i,\out,j}=\lambda_{i,j}$ for all
		$i,j\in[m]$ and $t\ge m+1$. Moreover, for every labeling
		$\lambda' \in \{0,1\}^{m \times m}$ with encoding $\theta'$, we have
		$f^t_{\theta,i}=f^t_{\theta',i}$ for all $i\in[m]$ and $t\in\{0,1\}$.
	\end{lemma}
	\begin{proof}
		Fix $i\in[m]$. We prove the two assertions separately starting from the second. 
		By \eqref{eq:init} and \eqref{eq:row},
		$f^t_{\theta,i,\row,j}=\ind{j=i}$ for all $t\ge0$ and $j\in[m]$.
		Equations~\eqref{eq:test}--\eqref{eq:last} give
		\begin{equation}\label{eq:first-state}
			\begin{aligned}
				f^1_{\theta,i,\test}&=\ind{\theta_i\ge0}=1,\\
				f^1_{\theta,i,\out,j}&=0 &&(j\in[m]),\\
				f^1_{\theta,i,\ptr,j}&=\ind{j=m} &&(j\in[m+1]).
			\end{aligned}
		\end{equation}
		Thus, all coordinates of $f^0_{\theta,i}$ and $f^1_{\theta,i}$
		are determined by $i$ alone, independently of $\lambda$.
		
		We now prove the first claim. 
		First, \eqref{eq:init} and \eqref{eq:ptr}--\eqref{eq:last}
		give, by induction on $t$,
		\begin{equation}\label{eq:pointer}
			f^t_{\theta,i,\ptr,j}=\ind{t=m+1-j}
			\qquad(j\in[m+1],\ t\ge0).
		\end{equation}
		Indeed, the formula holds at $t=0$; for $j\in[m]$,
		$f^{t+1}_{\theta,i,\ptr,j}
		=f^t_{\theta,i,\ptr,j+1}=\ind{t=m-j}$,
		while $f^t_{\theta,i,\ptr,m+1} = 0$ for $t\ge1$.
		We prove by descending induction on $k=m,\ldots,1$ that,
		for $t=m-k+2$ it holds that 
		\begin{equation}\label{eq:induct}
			f^{t-1}_{\theta,i,\out,j}=
			\begin{cases}
				\lambda_{i,j},&j>k,\\
				0,&j\le k,
			\end{cases}
		\end{equation}
	and that round $t$ satisfies
	$f^t_{\theta,i,\out,k}=\lambda_{i,k}$ and
	$f^t_{\theta,i,\out,j}=f^{t-1}_{\theta,i,\out,j}$
	for every $j\ne k$.
		
		\emph{Base case ($k=m$).}
		Here $t=2$, and \eqref{eq:first-state} gives
		$f^1_{\theta,i,\out,j}=0$ for every $j\in[m]$,
		as required by \eqref{eq:induct}.
		
		\emph{Induction step.}
		Assume \eqref{eq:induct} for $k$. By \eqref{eq:pointer},
		the only nonzero pointer bit at time $t-1$ is at position $k$.
		Hence,
		\[
		\sum_{j=1}^m2^{j-1}
		\bigl(f^{t-1}_{\theta,i,\out,j}
		+f^{t-1}_{\theta,i,\ptr,j}\bigr)
		=2^{k-1}+\sum_{j=k+1}^m2^{j-1}\lambda_{i,j}.
		\]
		Substituting into \eqref{eq:test} gives
		\[
		f^t_{\theta,i,\test}
		=\mathbf{1}\left\{2^{k-1}(\lambda_{i,k}-1)
			+\sum_{j=1}^{k-1}2^{j-1}\lambda_{i,j}\ge0\right\}
		=\lambda_{i,k},
		\]
		where the last equality follows since
		$\sum_{j=1}^{k-1}2^{j-1}\lambda_{i,j}<2^{k-1}$.
		Therefore, by \eqref{eq:out},
		\[
		f^t_{\theta,i,\out,k}
		=\ind{1+f^t_{\theta,i,\test}\ge2}
		=\lambda_{i,k}.
		\]
		For every $j\ne k$, \eqref{eq:pointer} gives
		$f^{t-1}_{\theta,i,\ptr,j}=0$, so
		\[
		f^t_{\theta,i,\out,j}
		=\ind{2f^{t-1}_{\theta,i,\out,j}
			+f^t_{\theta,i,\test}\ge2}
		=f^{t-1}_{\theta,i,\out,j}.
		\]
		For $k>1$, this establishes \eqref{eq:induct} for
		$k-1$, completing the induction.
		At time $m+1$, all output bits equal their labels.
		By \eqref{eq:pointer},
		$f^t_{\theta,i,\ptr,j}=\ind{t=m+1-j}=0$
		for all $j\in[m+1]$ and $t\ge m+1$.
		Thus, \eqref{eq:out} gives
		$f^{t+1}_{\theta,i,\out,j}=f^t_{\theta,i,\out,j}$
		for all $j\in[m]$ and $t\ge m+1$.
		Hence, $f^t_{\theta,i,\out,j}=\lambda_{i,j}$
		for every $t\ge m+1$, proving the lemma.
	\end{proof}

	\section{Implementing the Recursion with One LTF}\label{sec:single}
	
	The recursive formula described in the previous section updates several bits per round, whereas
	\eqref{eq:model} produces one bit per step. To reduce the recursive formula to one LTF, in this section we encode each vector of bits
	$f^t_{\theta,i}$ as a {\em block} of $p$ bits, where
	\begin{equation}\label{eq:params}
		p=378(m+2).
	\end{equation}
	We choose the coefficients of a single LTF so that the current state, encoded in a block $E^t_{\theta,i}$ (defined below), produces $E^{t+1}_{\theta,i}$ as its next block. 
	We now define the blocks. For the entire section, fix $\lambda \in \{0,1\}^{m \times m}$ and its encoding $\theta$. 
	Number the positions of each block by $0,\ldots,p-1$, and define positions within the block by
	\begin{equation}\label{eq:positions}
		\begin{aligned}
			&\pi(\test)=p/6+3,\\
			&\pi(\row,j)=p/6+63j+6 \qquad(j\in[m]),\\
			&\pi(\out,j)=p/6+63j+15 \qquad(j\in[m]),\\
			&\pi(\ptr,j)=p/6+63j+45+6(j\bmod2) \qquad(j\in[m+1]).
		\end{aligned}
	\end{equation}
	These positions are distinct and lie in $(p/6,p/3)$; the largest
	is at most $p/3-12$. Each row bit will be stored once at \(\pi(\row,j)\), while each test, output, and pointer bit will be stored twice, at \(\pi(b)\) and \(\pi(b)+1\), as specified in \(\eqref{eq:encode}\).
	The bit at position $0$ is fixed to $1$.
	Now, formally, define the block $E^t_{\theta,i}\in\{0,1\}^p$ using the above positions by
	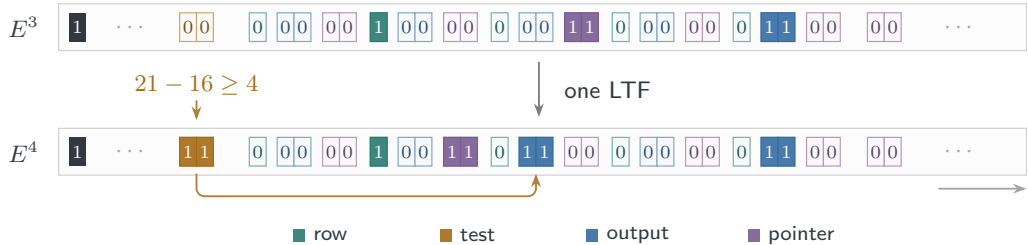
\begin{figure}[h]
		\centering
		\begingroup
		\definecolor{blockrow}{HTML}{287A70}
		\definecolor{blocktest}{HTML}{A66A15}
		\definecolor{blockout}{HTML}{356BA3}
		\definecolor{blockptr}{HTML}{795A92}
		\definecolor{blockink}{HTML}{333840}
		\begin{tikzpicture}[
			x=.97cm,y=1cm,font=\small,text=blockink,
			bit/.style={line width=.45pt},
			digit/.style={font=\scriptsize,inner sep=0pt},
			flow/.style={-{Stealth[length=1.8mm,width=1.3mm]},
				line width=.7pt}
			]
			\foreach \y in {0,-1.65}{
				\filldraw[fill=black!1,draw=black!18,line width=.45pt]
				(0,\y-.31) rectangle (13.2,\y+.31);
				\fill[blockink] (.14,\y-.19) rectangle (.37,\y+.19);
				\node[text=white,font=\scriptsize,inner sep=0pt]
				at (.255,\y) {$1$};
				\foreach \x in {1.0,12.3}
				\node[text=black!35,inner sep=0pt] at (\x,\y) {$\cdots$};
				\foreach \x/\c/\n in {
					1.65/blocktest/2,
					2.60/blockrow/1,2.98/blockout/2,3.60/blockptr/2,
					4.25/blockrow/1,4.63/blockout/2,5.25/blockptr/2,
					5.90/blockrow/1,6.28/blockout/2,6.90/blockptr/2,
					7.55/blockrow/1,7.93/blockout/2,8.55/blockptr/2,
					9.20/blockrow/1,9.58/blockout/2,10.20/blockptr/2,
					11.03/blockptr/2}{
					\foreach \k in {1,...,\n}{
						\filldraw[bit,draw=\c!60,fill=\c!4]
						({\x+.23*(\k-1)},\y-.19) rectangle ++(.23,.38);
						\node[digit,text=\c!70!black]
						at ({\x+.23*(\k-.5)},\y) {$0$};
					}
				}
				\filldraw[bit,draw=blockrow,fill=blockrow!90]
				(4.25,\y-.19) rectangle ++(.23,.38);
				\node[digit,text=white] at (4.365,\y) {$1$};
				\foreach \x in {9.58,9.81}{
					\filldraw[bit,draw=blockout,fill=blockout!90]
					(\x,\y-.19) rectangle ++(.23,.38);
					\node[digit,text=white] at (\x+.115,\y) {$1$};
				}
			}
			\foreach \x/\y in {6.90/0,7.13/0,5.25/-1.65,5.48/-1.65}{
				\filldraw[bit,draw=blockptr,fill=blockptr!90]
				(\x,\y-.19) rectangle ++(.23,.38);
				\node[digit,text=white] at (\x+.115,\y) {$1$};
			}
			\foreach \x/\c in {
				1.65/blocktest,1.88/blocktest,6.28/blockout,6.51/blockout}{
				\filldraw[bit,draw=\c,fill=\c!90]
				(\x,-1.84) rectangle ++(.23,.38);
				\node[digit,text=white] at (\x+.115,-1.65) {$1$};
			}
			
			\node[anchor=east,inner sep=0pt] at (-.26,0) {$E^3$};
			\node[anchor=east,inner sep=0pt] at (-.26,-1.65) {$E^4$};
			\node[text=blocktest,inner sep=0pt] at (1.88,-.76) {$21-16\ge4$};
			\draw[flow,draw=blocktest!85] (1.88,-.97) -- (1.88,-1.23);
			\draw[flow,draw=blockink!65] (6.55,-.46) -- (6.55,-1.19)
			node[midway,right=6pt,font=\small\sffamily] {one LTF};
			
			\draw[flow,draw=blocktest!85,rounded corners=3pt]
			(1.88,-1.94) -- (1.88,-2.24) -- (6.51,-2.24) -- (6.51,-1.94);
			\draw[flow,draw=blockink!45] (12.0,-2.14) -- (13.2,-2.14);
			
			\foreach \x/\c/\name in {
				3.2/blockrow/row,5.2/blocktest/test,
				7.2/blockout/output,9.4/blockptr/pointer}{
				\fill[\c!90] (\x,-2.84) rectangle ++(.16,.16);
				\node[anchor=west,font=\footnotesize\sffamily,inner sep=0pt]
				at (\x+.27,-2.76) {\name};
			}
		\end{tikzpicture}
		\endgroup
		\caption{Transition from $E^3$ to $E^4$ for $m=5$, $i=2$,
			and $\theta_2=21$.
			The black $1$ is the first bit of each block, fixed to $1$
			by the definition of the encoding in \eqref{eq:encode}.
			The same LTF generates all bits of $E^4$ from left to right.
			The row bits remain unchanged.
			The output bits initially represent $16$, and the pointer selects
			$j=3$, with weight $2^{3-1} = 4$.
			Since $21-16\ge4$, the new test bit is $1$.
			As the lower arrow indicates, this new test bit sets the third
			output bit to $1$, so the output bits now represent $16+4 = 20$.
			The pointer changes to $j=2$.
			Test, output, and pointer bits are each stored twice in adjacent
			positions. Unmarked positions are zero; spacing is schematic.}
		\label{fig:block-transition}
	\end{figure}
	
	\begin{subequations}\label{eq:encode}
		\begin{align}
			&E^t_{\theta,i,0}=1,\notag\\
			&E^t_{\theta,i,\pi(\row,j)}=f^t_{\theta,i,\row,j} \qquad(j\in[m]),\label{eq:enc-row}\\
			&E^t_{\theta,i,\pi(\test)}=E^t_{\theta,i,\pi(\test)+1}=f^t_{\theta,i,\test},\notag\\
			&E^t_{\theta,i,\pi(\out,j)}=E^t_{\theta,i,\pi(\out,j)+1}=f^t_{\theta,i,\out,j} \qquad(j\in[m]),\notag\\
			&E^t_{\theta,i,\pi(\ptr,j)}=E^t_{\theta,i,\pi(\ptr,j)+1}=f^t_{\theta,i,\ptr,j} \qquad(j\in[m+1]),\notag
		\end{align}
	\end{subequations}
	and set all other positions to zero. In particular, every encoded
	block ends in zero. We will use the single recurrence
	\begin{equation}\label{eq:scaled}
		z_n=\ind{-B+\sum_{\ell=1}^{p+1}a_\ell z_{n-\ell}\ge0}.
	\end{equation}
	For $d\ge p+1$, this is \eqref{eq:model} with
	$w_\ell=a_\ell/B$ for $\ell\le p+1$ and $w_\ell=0$ otherwise. The next lemma shows that there are coefficients $B, (a_{\ell})_{\ell \in [p+1]}$ that can faithfully update the blocks as defined above. 
	For $i\in[m]$ and $t\ge0$, the \emph{history}
	$(0,E^t_{\theta,i})$ denotes the initial values
	$z_{-p-1}=0$ and $z_{-p+r}=E^t_{\theta,i,r}$
	for $0\le r<p$.

	\Needspace{14\baselineskip}
	\begin{lemma}\label{lem:simulation}
		For every $\theta\in\{0,\ldots,2^m-1\}^m$, there are integers
		$B>0$ and $a_1,\ldots,a_{p+1}$ such that, for every $i\in[m]$
		and $t\ge0$, the history $(0,E^t_{\theta,i})$ generates
		$E^{t+1}_{\theta,i}$ as its next $p$ bits under \eqref{eq:scaled}.
	\end{lemma}
	
	\begin{proof}
		Set $B=2^{m+5}(m+2)$. We first define an auxiliary weight vector $a^{(0)}$ to be used for defining the final weight vector $a$.  Define integer coefficients
		$a^{(0)}_1,\ldots,a^{(0)}_{p+1}$ as follows, for every $j\in[m]$:
		\begin{subequations}\label{eq:coefficients}
			\begin{align}
				&a^{(0)}_{p+\pi(\test)-\pi(\row,j)}=4\theta_j,\label{eq:a0-test-row}\\
				&a^{(0)}_{p+\pi(\test)-\pi(\out,j)}=-2^{j+1},\label{eq:a0-test-out}\\
				&a^{(0)}_{p+\pi(\test)-\pi(\ptr,j)}=-2^{j+1},\label{eq:a0-test-ptr}\\
				&a^{(0)}_{\pi(\out,j)-\pi(\test)}=1,\label{eq:a0-out-test}\\
				&a^{(0)}_{p+\pi(\out,j)-\pi(\ptr,j)}=1,\label{eq:a0-out-ptr}\\
				&a^{(0)}_{p+\pi(\ptr,j)-\pi(\ptr,j+1)}=4 \qquad(j\in[m]).\label{eq:a0-ptr}
			\end{align}
		\end{subequations}
		For every $j\in[m]$, also set
		
		\begin{equation}\label{eq:coefficients2}
			\begin{aligned}
				&a^{(0)}_{p+\pi(\test)-\pi(\out,j)+1}=2^{j+1},\\
				&a^{(0)}_{p+\pi(\test)-\pi(\ptr,j)+1}=2^{j+1},\\
				&a^{(0)}_{\pi(\out,j)-\pi(\test)+1}=-1,\\
				&a^{(0)}_{p+\pi(\out,j)-\pi(\ptr,j)+1}=-1,\\
				&a^{(0)}_{p+\pi(\ptr,j)-\pi(\ptr,j+1)+1}=-4.
			\end{aligned}
		\end{equation}
		
		Set all other coefficients $a^{(0)}_\ell$ for $\ell \in [p+1]$ not defined in \eqref{eq:coefficients} and \eqref{eq:coefficients2} to zero.
		We show that the above assignment is consistent. 
		
		\medskip\noindent
		\textbf{Claim 1.} For any $\ell\in[p+1]$, if \eqref{eq:coefficients}--\eqref{eq:coefficients2}
		assign both $a^{(0)}_\ell=q$ and $a^{(0)}_\ell=q'$, then $q=q'$.
		
		\noindent\emph{Proof.}
		By \eqref{eq:positions}, all defined indices lie in
		$(0,p/6)\cup(5p/6,p)$, and
		\begin{equation}\label{eq:position-residues}
			(\pi(\test),\pi(\row,j),\pi(\out,j),\pi(\ptr,j))
			\equiv(3,6,15,45+6(j\bmod2))\pmod{63}.
		\end{equation}
		For $R=\{3,6,15,45,51\}$, the $20$ differences $u-v$ with $u,v\in R$ and $u\ne v$ modulo $63$ are precisely
		$\{3,6,\ldots,60\}$. Thus, for $u,v,u',v'\in R$
		with $u\ne v$ and $u'\ne v'$ it holds that 
		\begin{equation}\label{eq:residue-unique}
			u-v\equiv u'-v'\pmod{63}
			\implies (u,v)=(u',v').
		\end{equation}
		Since $63\mid p$, the indices in the six lines of
\eqref{eq:coefficients} have residues modulo $63$ belonging,
respectively, to
$\{60\}$, $\{51\}$, $\{15,21\}$, $\{12\}$,
$\{27,33\}$, and $\{6,57\}$.
These sets are pairwise disjoint, so indices from different
lines cannot coincide.
		Within \eqref{eq:a0-test-row}--\eqref{eq:a0-out-test}, equality
		of indices implies $j=j'$ by \eqref{eq:positions}.
		Within \eqref{eq:a0-out-ptr} and \eqref{eq:a0-ptr}, the assigned
		values are $1$ and $4$, respectively, even when indices repeat:
		\begin{subequations}
			\begin{align}
				&\pi(\out,j)-\pi(\ptr,j)=-30-6(j\bmod2),
				\label{eq:out-ptr-difference}\\
				&\pi(\ptr,j)-\pi(\ptr,j+1)=-69+12(j\bmod2).
				\label{eq:ptr-ptr-difference}
			\end{align}
		\end{subequations}
		Finally, all indices in \eqref{eq:coefficients} are $0$ modulo $3$.
		Each equality in \eqref{eq:coefficients2} is obtained from
		\eqref{eq:a0-test-out}--\eqref{eq:a0-ptr} by adding $1$ to the
		index and negating the value. These indices are $1$ modulo $3$,
		so they do not coincide with those in \eqref{eq:coefficients}
		and cannot assign conflicting values either.
		\hfill$\blacksquare$
		
		\medskip
		
		We now define the vector of coefficients $a$ based on $a^{(0)}$.
		Throughout the remainder of this proof, $b$ ranges over $\test$,
		$(\out,j)$ for $j\in[m]$, and $(\ptr,j)$ for $j\in[m+1]$. 
		Start with $a_\ell=a^{(0)}_\ell$, and for the following coordinates update 
		\begin{equation}\label{eq:remaining}
			\begin{aligned}
				&a_1=2,\\
				&a_{p-1}=1-B,\\
				&a_p=B+1,\\
				&a_{p+1}=-B-1,\\
				&a_{\pi(\test)}=B+1,\\
				&a_{\pi(\out,j)}=B-2 \qquad(j\in[m]),\\
				&a_{\pi(\ptr,j)}=B-3 \qquad(j\in[m+1]),\\
				&a_{\pi(b)+1}=B-1.
			\end{aligned}
		\end{equation}
		and
		\begin{equation}\label{eq:copies-suppression}
			a_\ell=-B^2\qquad(p/3<\ell\le2p/3).
		\end{equation}
		These assignments are disjoint from each other and from
		\eqref{eq:coefficients}--\eqref{eq:coefficients2}, by their index ranges and residues modulo $3$.
		Fix $i\in[m]$ and $t\ge0$. Let
		\begin{equation*}
			\widehat E=(0,E^t_{\theta,i},E^{t+1}_{\theta,i}),
		\end{equation*}
		indexed by $-p-1,\ldots,p-1$, and define the set of
		positions specified in \eqref{eq:encode} by
		\begin{equation*}
			\Pi=\{0\}\cup\{\pi(\row,j):j\in[m]\}
			\cup\bigcup_b\{\pi(b),\pi(b)+1\},
		\end{equation*}
		where $b$ ranges over $\test$, $(\out,j)$ for $j\in[m]$,
		and $(\ptr,j)$ for $j\in[m+1]$. Recall that all other positions are zero.
		For $0\le r<p$, define
		\begin{equation}\label{eq:score-definition}
			\sigma_r=-B+\sum_{\ell=1}^{p+1}a_\ell\widehat E_{r-\ell}.
		\end{equation}
		
		\medskip\noindent
		\textbf{Claim 2.} $\widehat E_r=\ind{\sigma_r\ge0}$ for every $r\in\Pi$.
		
		\noindent\emph{Proof.}
		We first compute the contribution of the coefficients
		$a^{(0)}_\ell$ to $\sigma_r$ in \eqref{eq:score-definition}. For $0\le r<p$, set
		\begin{equation}\label{eq:aux-score}
			\sigma^{(0)}_r=\sum_{\ell=1}^{p+1}a^{(0)}_\ell\widehat E_{r-\ell}.
		\end{equation}
		For the index calculations below, pad the vector $a^{(0)}$ with additional zero entries by setting
		$a^{(0)}_\ell=0$ for $\ell\in\mathbb Z\setminus[p+1]$.
		Equations~\eqref{eq:out-ptr-difference}--\eqref{eq:ptr-ptr-difference},
		\eqref{eq:residue-unique}, and
		\eqref{eq:a0-out-ptr}--\eqref{eq:a0-ptr} give
		\begin{subequations}
			\begin{align}
				&a^{(0)}_{p+\pi(\out,j)-\pi(\ptr,k)}=\ind{k=j}
				\qquad(j\in[m],\ k\in[m+1]),\label{eq:out-input-coefficient}\\
				&a^{(0)}_{p+\pi(\ptr,j)-\pi(\ptr,k)}=4\ind{k=j+1}
				\qquad(j,k\in[m+1]).\label{eq:ptr-input-coefficient}
			\end{align}
		\end{subequations}
		Also, \eqref{eq:encode} gives
		\begin{equation}\label{eq:secondary-zero}
			\widehat E_{r-\ell}=0
			\qquad(r\in\Pi,\ r\equiv0\pmod3,\ \ell\in[p+1],\ \ell\equiv1\pmod3).
		\end{equation}
		Thus, \eqref{eq:coefficients2} contributes zero at all positions
		$\pi(\test)$, $\pi(\row,j)$, $\pi(\out,j)$, and $\pi(\ptr,j)$.
		Substituting \eqref{eq:encode} into \eqref{eq:aux-score}, using
		\eqref{eq:residue-unique}, gives the following identities:
		\eqref{eq:a0-test-row}--\eqref{eq:a0-test-ptr} give
		\begin{equation}\label{eq:aux-test}
			\sigma^{(0)}_{\pi(\test)}
			=4\sum_{j=1}^m\theta_j f^t_{\theta,i,\row,j}
			-4\sum_{j=1}^m2^{j-1}
			(f^t_{\theta,i,\out,j}+f^t_{\theta,i,\ptr,j});
		\end{equation}
		\eqref{eq:a0-out-test} and \eqref{eq:out-input-coefficient} give
		\begin{equation*}
			\sigma^{(0)}_{\pi(\out,j)}
			=f^{t+1}_{\theta,i,\test}+f^t_{\theta,i,\ptr,j}
			\qquad(j\in[m]);
		\end{equation*}
		and \eqref{eq:ptr-input-coefficient} gives
		\begin{equation}\label{eq:aux-ptr}
			\sigma^{(0)}_{\pi(\ptr,j)}
			=4\sum_{k=1}^{m+1}\ind{k=j+1}f^t_{\theta,i,\ptr,k}
			\qquad(j\in[m+1]).
		\end{equation}
		In particular, $\sigma^{(0)}_{\pi(\ptr,m+1)}=0$.
		
		We now show that the contribution of $a^{(0)}$ to
		$\sigma_{\pi(b)+1}$ is zero for every coordinate $b$
		stored at both $\pi(b)$ and $\pi(b)+1$ in \eqref{eq:encode}.
		For every equality $a^{(0)}_\ell=q$ in
		\eqref{eq:a0-test-out}--\eqref{eq:a0-ptr},
		\eqref{eq:encode} and \eqref{eq:residue-unique} give
		\begin{equation*}
			\widehat E_{\pi(b)+1-\ell}=\widehat E_{\pi(b)-\ell}.
		\end{equation*}
		Since $a^{(0)}_{\ell+1}=-a^{(0)}_\ell$ by \eqref{eq:coefficients2} it holds that
		\begin{equation*}
			a^{(0)}_\ell\widehat E_{\pi(b)+1-\ell}
			+a^{(0)}_{\ell+1}\widehat E_{\pi(b)-\ell}=0.
		\end{equation*}
		Equation~\eqref{eq:a0-test-row} contributes zero at $\pi(b)+1$
		by \eqref{eq:enc-row} and \eqref{eq:residue-unique}.
		At $r=0$, each index $\ell$ in
		\eqref{eq:coefficients}--\eqref{eq:coefficients2} satisfies
		$(-\ell)\bmod p\in(0,p/6)\cup(5p/6,p)$, so
		$\widehat E_{-\ell}=0$ by \eqref{eq:encode}.
		At $r=\pi(\row,j)$, \eqref{eq:position-residues}--\eqref{eq:residue-unique}
		and \eqref{eq:secondary-zero} give zero for every term. Therefore,
		\begin{equation}\label{eq:aux-zero}
			\sigma^{(0)}_0=\sigma^{(0)}_{\pi(\row,j)}
			=\sigma^{(0)}_{\pi(b)+1}=0 \qquad(j\in[m]).
		\end{equation}
		
		We next consider \eqref{eq:remaining}--\eqref{eq:copies-suppression}.
		By \eqref{eq:positions} and \eqref{eq:encode}, for $r\in\Pi$ it holds that 
		\begin{equation}\label{eq:control-inputs}
			\begin{aligned}
				&\widehat E_{r-\pi(b)-e}=\ind{r=\pi(b)+e}
				\qquad(e\in\{0,1\}),\\
				&\widehat E_{r-\ell}=0
				\qquad(p/3<\ell\le2p/3).
			\end{aligned}
		\end{equation}
		Substituting \eqref{eq:aux-test}--\eqref{eq:aux-ptr},
		\eqref{eq:aux-zero}, and \eqref{eq:control-inputs} into
		\eqref{eq:score-definition} gives
		\begin{subequations}
			\begin{align}
				&\sigma_0=1,\label{eq:score-marker}\\
				&\sigma_{\pi(\row,j)}=(B+1)f^t_{\theta,i,\row,j}-B
				\qquad&&(j\in[m]),\notag\\
				&\sigma_{\pi(\test)}=1+2f^t_{\theta,i,\test}
				+\sigma^{(0)}_{\pi(\test)},\label{eq:score-test}\\
				&\sigma_{\pi(\out,j)}=2f^t_{\theta,i,\out,j}
				+f^t_{\theta,i,\ptr,j}+f^{t+1}_{\theta,i,\test}-2
				\qquad&&(j\in[m]),\notag\\
				&\sigma_{\pi(\ptr,j)}=2f^t_{\theta,i,\ptr,j}
				+4f^t_{\theta,i,\ptr,j+1}-3
				\qquad&&(j\in[m]),\notag\\
				&\sigma_{\pi(\ptr,m+1)}=2f^t_{\theta,i,\ptr,m+1}-3<0,
				\notag\\
				&\sigma_{\pi(b)+1}=2f^{t+1}_{\theta,i,b}-1.
				\label{eq:score-copy}
			\end{align}
		\end{subequations}
		By \eqref{eq:aux-test} we have $\sigma^{(0)}_{\pi(\test)}\in4\mathbb Z$.
		Since $1+2f^t_{\theta,i,\test}\in\{1,3\}$,
		\eqref{eq:score-test} and \eqref{eq:test} imply
		\begin{equation*}
			\ind{\sigma_{\pi(\test)}\ge0}
			=\ind{\sigma^{(0)}_{\pi(\test)}\ge0}
			=f^{t+1}_{\theta,i,\test}.
		\end{equation*}
		Observe that for any bits $u,v$ it holds that  $\ind{2u+4v-3\ge0}=v$.
                Thus, \eqref{eq:score-marker}--\eqref{eq:score-copy} give
$
\ind{\sigma_0\ge0}=1$,
$\ind{\sigma_{\pi(\row,j)}\ge0}=f^{t+1}_{\theta,i,\row,j},
$
and, for every test, output, or pointer coordinate $b$ we have
\[
\ind{\sigma_{\pi(b)}\ge0}
=\ind{\sigma_{\pi(b)+1}\ge0}
=f^{t+1}_{\theta,i,b}.
\]
Hence, by \eqref{eq:encode} it holds that
$
\ind{\sigma_r\ge0}=\widehat E_r$ for every $r\in\Pi$.
		This proves the claim.\hfill$\blacksquare$

		\medskip
		
		\Needspace{6\baselineskip}
		We show that the scores at all remaining positions are negative. 
		
		\medskip\noindent
		\textbf{Claim 3.} $\sigma_r<0$ for every $0\le r<p$ with $r\notin\Pi$.
		
		\noindent\emph{Proof.}
		By \eqref{eq:coefficients}--\eqref{eq:coefficients2},
		for each $j\in[m]$ the assigned absolute values are
		$4\theta_j$, four times $2^{j+1}$, four times $1$,
		and twice $4$. By Claim~1, repeated indices define the same coefficient, rather
than adding their assigned values. Summing the absolute values of
all assignments therefore gives an upper bound:
\begin{equation*}
\begin{aligned}
\sum_{\ell=1}^{p+1}|a^{(0)}_\ell|
&\le \sum_{j=1}^m
\left(4\theta_j+4\cdot 2^{j+1}+12\right)\\
&=4\sum_{j=1}^m\theta_j+16(2^m-1)+12m\\
&<(7m+16)2^m<B/2.
\end{aligned}
\end{equation*}
		and
		\begin{equation*}
			\sum_{\ell=1}^{p+1}\max\{a_\ell,0\}
			<(4m+6)(B+1)+B/2<B^2.
		\end{equation*}
		If $p/3<r\le2p/3$, then $a_r=-B^2$ and $\widehat E_0=1$,
		so
		\begin{equation*}
			\sigma_r
			\le-B-B^2+\sum_{\ell=1}^{p+1}\max\{a_\ell,0\}<0.
		\end{equation*}
		Otherwise,
		\begin{equation*}
			\widehat E_{r-p}
			=\widehat E_{r-\pi(b)}
			=\widehat E_{r-\pi(b)-1}=0.
		\end{equation*}
		For the last two equalities, a nonzero entry would require
		$r=\pi(b)+e\in\Pi$, or
		$r\in(p/3,2p/3)$; a negative index instead corresponds to
		a block position greater than $2p/3$, whose entry is zero.
		Thus,
		\begin{equation*}
			\sigma_r\le-B+\sum_{\ell=1}^{p+1}|a^{(0)}_\ell|+2
			<-B+B/2+2<0.
		\end{equation*}
		This proves the claim.\hfill$\blacksquare$
		
		\medskip
		By Claims 2 and 3 it holds that $\widehat E_r=\ind{\sigma_r\ge0}$ for every
		$0\le r<p$. Since $-p-1\le r-\ell<r$ for every
		$1\le\ell\le p+1$, induction on $r$ shows that
		\eqref{eq:scaled}, initialized with $(0,E^t_{\theta,i})$,
		generates exactly $E^{t+1}_{\theta,i}$.
		The coefficients depend only on $m$ and $\theta$.
	\end{proof}

	\section{Choosing the Inputs}\label{sec:inputs}
	
	We now construct $m^2$ fixed input strings and prove that, for every assignment
	of labels to these inputs, a single autoregressive LTF realizes all
	the labels. 
	For each fixed encoding $\theta$, write $E_i^t:=E^t_{\theta,i}$.
	By Lemma~\ref{lem:digits}, the blocks $E_i^0,E_i^1$ depend only on $i$,
	not on the labels. We construct inputs $x^{i,s}$ by selecting $d$ consecutive
	bits from this string which is independent of the labels, with initial zero padding as specified below. We choose $r_s\in\{0,\ldots,p-1\}$ so that
	$T+r_s\equiv1+\pi(\out,s)\pmod p$.
	Thus, the bit generated at time $T$ is at position
	$\pi(\out,s)$ of an encoded block.
	Assume $d\ge p+1$ and $T\ge(m+3)p$. For each $i\in[m]$, define the string
$E_i=((E_i)_n)_{n=1-d}^{p}$ by
\begin{equation}\label{eq:startup}
(E_i)_n=
\begin{cases}
0, & 1-d\le n\le -p,\\
(E_i^0)_{n+p-1}, & 1-p\le n\le 0,\\
(E_i^1)_{n-1}, & 1\le n\le p.
\end{cases}
\end{equation}
Equivalently, $E_i=0^{d-p}\,E_i^0\,E_i^1$, 
	 where $0^{d-p}$ denotes $d-p$ zeros.
	For each column $s\in[m]$, set
	\begin{equation}\label{eq:shift}
		r_s=(\pi(\out,s)-T+1)\bmod p,
	\end{equation}
	and define
	\begin{equation}\label{eq:prompts}
		x_n^{i,s}=(E_i)_{r_s+n}
		\qquad(1-d\le n\le0,\ i,s\in[m]).
	\end{equation}
	All indices lie between $1-d$ and $p-1$, so these $m^2$ prompts
	are well defined and independent of the labels.
	
	\begin{figure}[ht]
		\centering
		\begingroup
		\definecolor{inputblue}{HTML}{356BA3}
		\definecolor{inputink}{HTML}{333840}
		\begin{tikzpicture}[
			x=1cm,y=1cm,font=\small,text=inputink,
			flow/.style={-{Stealth[length=1.7mm,width=1.2mm]},
				line width=.65pt,draw=inputblue!80}
			]
			\node at (3.5,1.04) {$E_i^0$};
			\node at (6.5,1.04) {$E_i^1$};
			\node at (11.9,1.04) {$E_i^t$};
			
			\foreach \y/\shift/\k in {0/.45/1,-1.38/1.05/2,-2.76/1.65/3}{
				\fill[black!2] (0,\y-.25) rectangle (8,\y+.25);
				\fill[black!2] (10.4,\y-.25) rectangle (13.4,\y+.25);
				\fill[inputblue!8]
				(\shift,\y-.25) rectangle (\shift+5,\y+.25);
				\draw[black!25,line width=.45pt]
				(0,\y-.25) rectangle (8,\y+.25);
				\foreach \x in {2,5}
				\draw[black!25,line width=.45pt]
				(\x,\y-.25) -- (\x,\y+.25);
				\draw[black!25,line width=.45pt]
				(10.4,\y-.25) rectangle (13.4,\y+.25);
				\node[text=black!35,inner sep=0pt] at (1,\y) {$\cdots$};
				\node[text=black!35,inner sep=0pt] at (9.2,\y) {$\cdots$};
				\draw[inputblue,line width=.9pt]
				(\shift,\y-.31) rectangle (\shift+5,\y+.31);
				\node[text=inputblue,font=\normalsize,inner sep=0pt]
				at (\shift+2.5,\y+.53) {$x^{i,s_{\k}}$};
				
				\draw[draw=inputblue!45,line width=.4pt]
				(0,\y-.35) -- (0,\y-.66)
				(\shift,\y-.35) -- (\shift,\y-.66);
				\draw[{Stealth[length=1.1mm,width=.8mm]}-{Stealth[length=1.1mm,width=.8mm]},
				draw=inputblue!80,line width=.55pt]
				(0,\y-.59) -- (\shift,\y-.59);
				\node[text=inputblue,font=\small,inner sep=0pt,anchor=north]
				at (\shift/2,\y-.70) {$r_{s_{\k}}$};
				
				\filldraw[fill=inputblue!90,draw=inputblue,line width=.5pt]
				(\shift+10.32,\y-.20) rectangle ++(.16,.40);
				\draw[flow,rounded corners=2pt]
				(\shift+5,\y+.36) -- (\shift+5,\y+.66) --
				(\shift+10.4,\y+.66) -- (\shift+10.4,\y+.30);
			}
		\end{tikzpicture}
		\endgroup
		\caption{Choosing inputs for a fixed row $i$.
			The three rows correspond to columns $s_1,s_2,s_3$; each repeats the
			same string $0^{d-p}E_i^0E_i^1$ on the left.
			Each blue frame selects $d$ consecutive bits, forming $x^{i,s}$.
			The short arrows show the shift $r_s$ from the start of the string
			to the start of each input. The long arrows represent $T$
			generated bits.
			The blue cell is at position $\pi(\mathrm{out},s)$ in the later
			block $E_i^t$ and equals $\lambda_{i,s}$.
			Since $E_i^0$ and $E_i^1$ are independent of the labels, so are
			the selected inputs.
			Shown for $T$ a multiple of $p$; spacing is schematic.}
		\label{fig:input-selection}
	\end{figure}
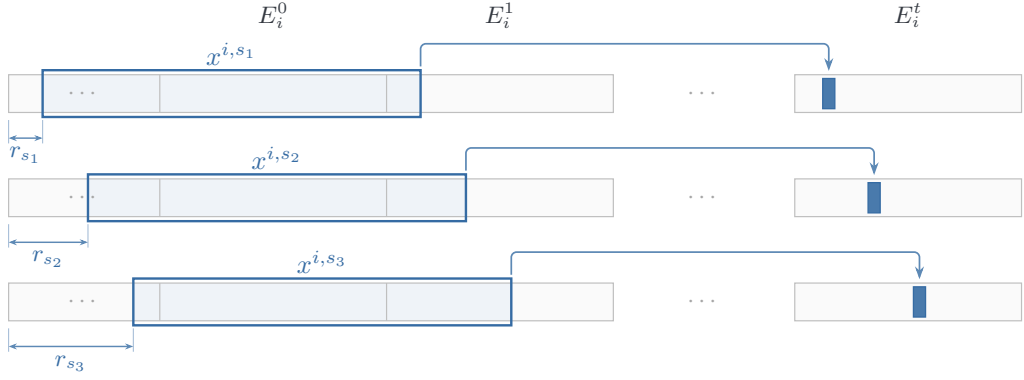

	\Needspace{5\baselineskip}
	\begin{lemma}\label{lem:prompts}
		For $d\ge p+1$ and $T\ge(m+3)p$, the $m^2$ prompts in
		\eqref{eq:prompts} are distinct and shattered by $\F_{d,T}$.
	\end{lemma}
	\begin{proof}
		Fix $\lambda\in\{0,1\}^{m\times m}$ and its encoding $\theta$.
		Take $B,a_1,\ldots,a_{p+1}$ from Lemma~\ref{lem:simulation}, and define
		\begin{equation}\label{eq:normalized}
			w_\ell=
			\begin{cases}
				a_\ell/B,&1\le\ell\le p+1,\\
				0,&p+1<\ell\le d.
			\end{cases}
		\end{equation}
		
		For each $i\in[m]$, define $z^{(i)}_n$ for $n\ge1-d$ by
		\begin{align}
			&(z^{(i)}_{1-d},\ldots,z^{(i)}_0)
			=0^{d-p}E_i^0,\label{eq:canonical-init}\\
			&z^{(i)}_n
			=\ind{\sum_{\ell=1}^d w_\ell z^{(i)}_{n-\ell}\ge1}
			\qquad(n\ge1).\label{eq:canonical-update}
		\end{align}
		We prove by induction on $t\ge0$ that
		\begin{equation}\label{eq:block-invariant}
			(z^{(i)}_{(t-1)p},\ldots,z^{(i)}_{tp})=(0,E_i^t).
		\end{equation}
		For $t=0$, this follows from \eqref{eq:canonical-init}
		and $d\ge p+1$.
		Assume \eqref{eq:block-invariant} holds for $t$.
		By \eqref{eq:normalized}, Lemma~\ref{lem:simulation}, and
		$(E_i^t)_{p-1}=0$ from \eqref{eq:encode}, we have
		\begin{equation*}
			(z^{(i)}_{tp},\ldots,z^{(i)}_{(t+1)p})
			=((E_i^t)_{p-1},E_i^{t+1})
			=(0,E_i^{t+1}).
		\end{equation*}
		This proves \eqref{eq:block-invariant}.
		By \eqref{eq:startup}, \eqref{eq:canonical-init}, and
		\eqref{eq:block-invariant} at $t=1$,
		\begin{equation}\label{eq:startup-agreement}
			z^{(i)}_n=(E_i)_n
			\qquad(1-d\le n\le p).
		\end{equation}
		Fix $i,s\in[m]$. Define
		\begin{align}
			&z^{(i,s)}_n=x^{i,s}_n
			\qquad(1-d\le n\le0),\notag\\
			&z^{(i,s)}_n
			=\mathbf{1} \left\{\sum_{\ell=1}^d w_\ell z^{(i,s)}_{n-\ell}\ge1\right\}
			\qquad(1\le n\le T).\label{eq:shifted-update}
		\end{align}
		Since $0\le r_s<p$, equations
		\eqref{eq:prompts} and \eqref{eq:startup-agreement} give
		\begin{equation*}
			z^{(i,s)}_n=x^{i,s}_n=(E_i)_{r_s+n}
			=z^{(i)}_{r_s+n}
			\qquad(1-d\le n\le0).
		\end{equation*}
		For $1\le n\le T$, assume this equality holds at every
		index from $1-d$ to $n-1$.
		Since $1-d\le n-\ell<n$ for $\ell\in[d]$ and $r_s+n\ge1$,
		\eqref{eq:shifted-update} and \eqref{eq:canonical-update} give
		\begin{equation*}
			z^{(i,s)}_n
			=\ind{\sum_{\ell=1}^d w_\ell z^{(i)}_{r_s+n-\ell}\ge1}
			=z^{(i)}_{r_s+n}.
		\end{equation*}
		Induction therefore yields
		\begin{equation}\label{eq:continuation}
			f^T_w(x^{i,s})=z^{(i,s)}_T=z^{(i)}_{T+r_s}.
		\end{equation}
		By \eqref{eq:shift},
		$t:=1+(T+r_s-1-\pi(\out,s))/p$ is an integer satisfying
		\begin{align}
			&T+r_s=(t-1)p+1+\pi(\out,s),\label{eq:h}\\
			&t-1=\frac{T+r_s-1-\pi(\out,s)}p
			\ge\frac{T-p}p\ge m+2.\notag
		\end{align}
		The inequalities use $r_s\ge0$, $\pi(\out,s)\le p-1$,
		and $T\ge(m+3)p$.
		Thus \eqref{eq:continuation}, \eqref{eq:h},
		\eqref{eq:block-invariant}, \eqref{eq:encode}, and
		Lemma~\ref{lem:digits} imply
		\begin{equation}\label{eq:output}
			f^T_w(x^{i,s})
			=(E_i^t)_{\pi(\out,s)}
			=f^t_{\theta,i,\out,s}
			=\lambda_{i,s}.
		\end{equation}
		The vector $w$ in \eqref{eq:normalized} is common to all $i,s$.
		By Lemma~\ref{lem:digits} and
		\eqref{eq:startup}--\eqref{eq:prompts}, the inputs do not depend
		on $\lambda$. 
		
		It remains only to show that all of these inputs are distinct. 
		Assume for a contradiction that $x^{i,s}=x^{i',s'}$ for distinct index pairs $(i,s) \neq (i',s')$; then,
		choose $\lambda_{i,s}=0$ and $\lambda_{i',s'}=1$.
		Equation~\eqref{eq:output} then gives the contradiction
		\begin{equation*}
			0=f^T_w(x^{i,s})=f^T_w(x^{i',s'})=1.
		\end{equation*}
		Hence, the $m^2$ inputs are distinct, and \eqref{eq:output}
		shows that they are shattered by $\F_{d,T}$.
	\end{proof}

	We now combine the preceding construction with the choice of $m$ to prove Theorem~\ref{thm:main}. 
	
	\begin{proof}[Proof of Theorem~\ref{thm:main}]
		Set
		\begin{equation}\label{eq:choosem}
			m=\max\left\{1,\left\lfloor\frac{d-1}{378}\right\rfloor-2\right\}.
		\end{equation}
		If $m\ge2$, equations \eqref{eq:params} and \eqref{eq:choosem} give
		\begin{align*}
			&p=378(m+2)\le d-1,\\
			&(m+3)p<p^2<d^2\le T.
		\end{align*}
		Lemma~\ref{lem:prompts} therefore gives
		$\operatorname{VCdim}(\F_{d,T})\ge m^2$.
		If $m=1$, \eqref{eq:model} gives
		\begin{equation}\label{eq:singleton}
			f^T_{\mathbf 0}(1,\ldots,1)=0,
			\qquad
			f^T_{(1,0,\ldots,0)}(1,\ldots,1)=1,
		\end{equation}
		so $\operatorname{VCdim}(\F_{d,T})\ge1=m^2$ in this case as well.
		
		Equation~\eqref{eq:choosem} implies $d\le1512$ when $m=1$,
		and $d\le378(m+3)\le1512m$ when $m\ge2$. Consequently,
		\begin{equation*}
			\operatorname{VCdim}(\F_{d,T})
			\ge m^2\ge\frac{d^2}{1512^2}\ge\frac{d^2}{3024^2}.
		\end{equation*}
		
		For $m\ge2$, the coefficient definitions
		\eqref{eq:coefficients}--\eqref{eq:coefficients2},
		\eqref{eq:remaining}, and \eqref{eq:copies-suppression} give $\max_{1\le\ell\le p+1}|a_\ell|\le B^2$ and 
		\begin{align*}
			&\log_2 B=m+5+\log_2(m+2)=O(m+\log m).
		\end{align*}
		Thus, each weight in \eqref{eq:normalized} has a numerator and
		denominator of $O(m+\log m)$ bits. Since $m=O(d)$, the weights
		have polynomial bit length in $d$. For $m=1$, the two vectors
		in \eqref{eq:singleton} have entries in $\{0,1\}$.
	\end{proof}

\section{Stochastic Digit Extraction}\label{sec:stochastic-digits}

We prove Theorem~\ref{thm:stochastic} in three steps.
First, we show how repeated independent random bits can recover
any of $2^m$ labels encoded in a single real number.
Next, we extend Lemma~\ref{lem:simulation} to implement this procedure
with one logistic generator, while controlling the simulation errors.
Finally, we construct $2^m$ fixed, prompts independent of the labels
whose final output probabilities realize every labeling
at a common horizon.
Fix $m\ge2$ and put $N=2^m$. For arbitrary labels
$\lambda_0,\ldots,\lambda_{N-1}\in\{0,1\}$, encode
\begin{equation}\label{eq:st-code}
 \alpha=4^{-2N}+\sum_{j=0}^{N-1}2(1-\lambda_j)4^{-(j+1)}.
\end{equation}
Then $0<\alpha<1$. For $x=(x_0,\ldots,x_{m-1})\in\{0,1\}^m$,
define its {\em address}
\begin{equation}\label{eq:st-address}
 j(x)=\sum_{r=0}^{m-1}2^r x_r.
\end{equation}
Set $K=4^{2N}$ and $\beta=\log(2\pi\alpha/K)$, where $\log$ is the natural logarithm.
The following odds give an affine logistic sampling probability:
\begin{align}
 a_x&=\frac{2\pi\alpha}{K}4^{j(x)},\label{eq:st-odds}\\
 p_x&=\frac{a_x}{1+a_x}=\sigma\bigl(\beta+(\log4)j(x)\bigr).
 \label{eq:st-probability}
\end{align}
For each fixed $x\in\{0,1\}^m$, let
$f^t_{\test}\in\{0,1\}$, $t\ge1$, be independent
$\operatorname{Bernoulli}(p_x)$ random variables and define
\begin{align}
 S_t&=\sum_{s=1}^t f^s_{\test},\label{eq:st-sum}\\
 Y_t&=\ind{S_t\bmod4\in\{0,1\}}.
 \label{eq:st-reader}
\end{align}
Set $\delta_N=\tfrac14 e^{-2\cdot4^N}$.
The next lemma shows that the probability of the final output
being $1$ lies above or below $1/2$, according to the selected
label, with a gap of at least $\delta_N$.

\begin{lemma}\label{lem:st-reader}
 For every $x\in\{0,1\}^m$ it holds that
 $(2\lambda_{j(x)}-1)(\Pr(Y_K=1)-1/2)\ge\delta_N$.
\end{lemma}
\begin{proof}
 In \eqref{eq:st-code},
the terms preceding digit $j$ become integers after multiplication
by $4^j$. Since $\sum_{h\ge2}2\cdot4^{-h}=1/6$, for $0\le j<N$, we have
\begin{equation}\label{eq:st-tail}
 0\le{4^j\alpha}-\lfloor 4^j\alpha\rfloor-\frac{1-\lambda_j}{2}
 \le\frac16+4^{-N-1}<\frac14.
\end{equation}
Thus, $2\pi4^j\alpha-\pi/4$ is within $\pi/4$ of $0$ when
$\lambda_j=1$ and of $\pi$ when $\lambda_j=0$, modulo $2\pi$.
Hence,
\begin{equation}\label{eq:st-cosine-code}
 (2\lambda_j-1)\cos(2\pi4^j\alpha-\pi/4)\ge1/\sqrt2.
\end{equation}
Let $\mathrm i^2=-1$. For integer $s$, the real part
$\operatorname{Re}((1-\mathrm i)\mathrm i^s)$ is $1$ on residues
$0,1$ and $-1$ on residues $2,3$. Independence gives
\begin{equation}\label{eq:st-fourier}
\begin{aligned}
 \Pr(Y_K=1)-\tfrac12
 &=\tfrac12\operatorname{Re}\bigl((1-\mathrm i)(1-p_x+\mathrm i p_x)^K\bigr)\\
 &=\frac1{\sqrt2}\left(\frac{\sqrt{1+a_x^2}}{1+a_x}\right)^K
     \cos(K\arctan a_x-\pi/4).
\end{aligned}
\end{equation}
For $a\ge0$ it holds that $0\le a-\arctan a\le a^3/3$. Thus,
since $Ka_x<2\cdot4^N$ by \eqref{eq:st-odds} we have
\begin{equation}\label{eq:st-angle-error}
 0\le2\pi\alpha4^{j(x)}-K\arctan a_x
 \le Ka_x^3/3<\tfrac83\,4^{-N}<\tfrac18.
\end{equation}
The cosine function is $1$-Lipschitz, so \eqref{eq:st-cosine-code} gives
\begin{equation}\label{eq:st-cosine-margin}
 (2\lambda_{j(x)}-1)\cos(K\arctan a_x-\pi/4)
 \ge1/\sqrt2-1/8>1/2.
\end{equation}
Also,
\begin{equation}\label{eq:st-amplitude}
 \left(\frac{\sqrt{1+a_x^2}}{1+a_x}\right)^K
 \ge(1+a_x)^{-K}\ge e^{-Ka_x}\ge e^{-2\cdot4^N}.
\end{equation}
Equations~\eqref{eq:st-fourier}, \eqref{eq:st-cosine-margin}, and
\eqref{eq:st-amplitude} give the proof.
\end{proof}

\subsection{Reusing the Deterministic Simulation}\label{sec:stochastic-single}

In this section, we first describe a recursion that generates the random bits
from the previous section and tracks their sum modulo $4$
using a constant number of threshold gates per round.
We then implement this recursion with a single logistic generator, analogous to our techniques in the deterministic setting.

For fixed $\alpha$ and $x\in\{0,1\}^m$, we represent the
recursion at round $t$ by a random vector
$f^t\in\{0,1\}^{m+7}$, whose coordinates have the following roles.
The $m$ coordinates $(\mathrm{in},j)$ store the input bits $x_j$,
and the coordinate $\test$ stores the newly sampled bit.
The two coordinates $(\mathrm{state},i)$, $i\in\{1,2\}$,
encode the number of sampled ones modulo $4$.
The auxiliary and output coordinates update the counter:
if the sampled bit is $0$, the state $(u,v)$ remains unchanged;
if it is $1$, the state becomes $(v,1-u)$.
These updates are computed using linear threshold functions.
We suppress $\alpha,x$ in the notation.
Initially, no bits have been sampled, and the counter value
$0$ is represented by $(0,1)$.
Thus, all coordinates of $f^0$ are zero except
\begin{equation}\label{eq:st-init}
	\begin{aligned}
		f^0_{\mathrm{in},j}&=x_j,\\
		f^0_{\mathrm{state},2}&=1.
	\end{aligned}
\end{equation}
Fix $G\ge1$. At round $t\ge1$, compute these scores in order,
immediately generating the corresponding bit by
\eqref{eq:st-sample}--\eqref{eq:st-gates}:
\begin{subequations}\label{eq:st-recursion}
\begin{align}
 g^t_{\test}&=\beta+(\log4)\sum_{j=0}^{m-1}2^j f^{t-1}_{\mathrm{in},j},
 \label{eq:st-sampler-score}\\
 g^t_{\mathrm{aux},1}&=G(f^{t-1}_{\mathrm{state},2}+f^t_{\test}-3/2),
 \label{eq:st-first-aux}\\
 g^t_{\mathrm{aux},2}&=G(f^t_{\test}-f^{t-1}_{\mathrm{state},1}-1/2),
 \label{eq:st-second-aux}\\
 g^t_{\out,1}&=G(f^{t-1}_{\mathrm{state},1}-f^t_{\test}+2f^t_{\mathrm{aux},1}-1/2),
 \label{eq:st-first-state}\\
 g^t_{\out,2}&=G(f^{t-1}_{\mathrm{state},2}-f^t_{\test}+2f^t_{\mathrm{aux},2}-1/2),
 \label{eq:st-output}\\
 g^t_{\mathrm{state},i}&=G(f^t_{\out,i}-1/2).
 \label{eq:st-state-copy}
\end{align}
\end{subequations}
Let $\mathcal V=\{\test\}\cup
\{(\mathrm{aux},i),(\out,i),(\mathrm{state},i):i\in\{1,2\}\}$
be the seven non-address coordinates.
At each round $t$, independently sample $f^t_{\test}=1$
with probability $\sigma(g^t_{\test})$.
Every other non-address bit $f^t_v$ for $v \in \mathcal{V} \setminus \{\test\}$ is determined by its score:
it equals $1$ if $g^t_v\ge0$, and $0$ otherwise.
The address bits remain unchanged. Formally,
\begin{align}
 f^t_{\test}&\sim\operatorname{Bernoulli}(\sigma(g^t_{\test})),
 \label{eq:st-sample}\\
 f^t_v&=\ind{g^t_v\ge0},\label{eq:st-gates}\\
 f^t_{\mathrm{in},j}&=f^{t-1}_{\mathrm{in},j}.
 \label{eq:st-address-copy}
\end{align}
Since the address bits remain equal to $x_j$,
\eqref{eq:st-sampler-score} and \eqref{eq:st-probability} give
$
\sigma(g^t_{\test})
=\sigma\bigl(\beta+(\log4)j(x)\bigr)=p_x.
$
Thus, the bits $f^t_{\test}$ are independent
$\operatorname{Bernoulli}(p_x)$ random variables, as required.

Let $(u_t,v_t)=(f^t_{\mathrm{state},1},f^t_{\mathrm{state},2})$
denote the counter state after round $t$.
Equations~\eqref{eq:st-first-aux}--\eqref{eq:st-state-copy}
leave this state unchanged when $f^t_{\test}=0$,
and change $(u,v)$ to $(v,1-u)$ when $f^t_{\test}=1$.
Starting from $(0,1)$, each sampled $1$ advances the state
along the cycle
$
(0,1)\to(1,1)\to(1,0)\to(0,0)\to(0,1).
$
Thus, the second coordinate equals $1$ exactly when the
number of sampled ones is $0$ or $1$ modulo $4$.
Consequently, by \eqref{eq:st-reader}
\begin{equation}\label{eq:st-state-invariant}
 f^t_{\out,2}=f^t_{\mathrm{state},2}=Y_t.
\end{equation}

\subsubsection*{Block Construction}
We now define the blocks and first the positions within each block, analogous to the deterministic case. 
Set $p=6144(m+1)$. For the indices $i,j$ above, define
\begin{equation}\label{eq:st-positions}
\begin{aligned}
 \pi(\test)&=p/6+4,\\
 \pi(\mathrm{aux},i)&=p/6+4\cdot2^i,\\
 \pi(\out,i)&=p/6+16\cdot2^i,\\
 \pi(\mathrm{state},i)&=p/6+64\cdot2^i,\\
 \pi(\mathrm{in},j)&=p/6+1024j+512.
\end{aligned}
\end{equation}
These are distinct multiples of $4$ in $(p/6,p/3)$.
As in \eqref{eq:encode}, define $E^t\in\{0,1\}^p$, indexed by
$0,\ldots,p-1$, by setting, for $v\in\mathcal V$ and $e\in\{0,1\}$,
\begin{equation}\label{eq:st-encode}
\begin{aligned}
 E^t_0&=1,\\
 E^t_{\pi(\mathrm{in},j)}&=f^t_{\mathrm{in},j},\\
 E^t_{\pi(v)+e}&=f^t_v,
\end{aligned}
\end{equation}
and all other entries to zero. Write $E^0_x$ for the initial
block; it is independent of the labels by \eqref{eq:st-init}.

\begin{lemma}\label{lem:st-simulation}
	For every $\alpha\in(0,1)$ and $G\ge1$, there exist
	$w\in\mathbb R^{p+1}$ and $b\in\mathbb R$ such that,
	for every $x\in\{0,1\}^m$, the process \eqref{eq:stochastic-model}
	initialized at $(0,E^0_x)$ generates blocks
	$\widehat E^1,\ldots,\widehat E^K$ satisfying
	\[
	\left|\Pr(\widehat E^K_{\pi(\out,2)}=1)
	-\Pr(Y_K=1)\right|\le pKe^{-G/2}.
	\]
\end{lemma}

\begin{proof}
The proof is an adaptation of Lemma~\ref{lem:simulation} for the stochastic setting.
For $v\in\mathcal V$, let $b_v$ be the constant term and
$c_{vhu}$ the coefficient of $f_u^{t-h}$ in
\eqref{eq:st-recursion}, where $h\in\{0,1\}$ and $u$ ranges
over all coordinates. For every nonzero $c_{vhu}$, let
\begin{equation}\label{eq:st-program}
 a^{(0)}_{hp+\pi(v)-\pi(u)}=c_{vhu}.
\end{equation}
For every $c_{vhu}\ne0$ with $u\in\mathcal V$, also define
\begin{equation}\label{eq:st-companions}
 a^{(0)}_{hp+\pi(v)-\pi(u)+1}=-c_{vhu}.
\end{equation}
All other entries of $a^{(0)}\in\mathbb R^{p+1}$ are zero.
By \eqref{eq:st-positions}, all coefficient indices defined in
\eqref{eq:st-program}--\eqref{eq:st-companions}
belong to $(0,p/6)\cup(5p/6,p)$.
The nonzero ordered differences of $\{2^i:2\le i\le9\}$ are
distinct modulo $1024$: their absolute values are at most $508$,
and $2^a-2^b=2^b(2^{a-b}-1)$ identifies $a>b$.
By \eqref{eq:st-positions}--\eqref{eq:st-companions},
every $\ell\in[p+1]$ with $a^{(0)}_\ell\ne0$
satisfies $0<\ell<p/6$ or $5p/6<\ell<p$.
Set
\begin{equation}\label{eq:st-control-size}
 B=16\left(1+3G+\sum_{\ell=1}^{p+1}|a^{(0)}_\ell|
                  +\sum_{v\in\mathcal V}|b_v|\right).
\end{equation}
We now define the coefficients $a_1,\ldots,a_{p+1}$
for the stochastic construction. Initially, set
$a_\ell=a^{(0)}_\ell$ for every $\ell\in[p+1]$,
where $a^{(0)}$ was defined in
\eqref{eq:st-program}--\eqref{eq:st-companions}.
Then overwrite the following coefficients:
\begin{equation}\label{eq:st-control}
\begin{aligned}
 a_1&=2G,\\
 a_{p-1}&=-(B+G),\\
 a_p&=B+G,\\
 a_{p+1}&=-(B+G),\\
 a_{\pi(v)}&=B+b_v,\\
 a_{\pi(v)+1}&=B-G.
\end{aligned}
\end{equation}
Finally, set
$
a_\ell=-B^2
$for every integer $p/3<\ell \leq 2p/3$,
as in the deterministic construction
\eqref{eq:copies-suppression}.
These indices do not overlap any coefficients assigned above,
so this completes the definition of the vector $a$ for the stochastic setting.

Fix $t\in[K]$. To analyze how the logistic generator produces
$E^t$ from $E^{t-1}$ we do the following.
For each position $0\le r<p$, define its logistic score by index $\widehat E=(0,E^{t-1},E^t)$ by
$-p-1,\ldots,p-1$, and define, for $0\le r<p$,
\begin{equation}\label{eq:st-score}
 s_r=-B+\sum_{\ell=1}^{p+1}a_\ell\widehat E_{r-\ell}.
\end{equation}
By \eqref{eq:st-encode}, for every coordinate $u$ and
$h\in\{0,1\}$,
\begin{equation}\label{eq:st-read-positions}
	\begin{aligned}
		\widehat E_{\pi(u)-hp}&=f^{t-h}_u,\\
		\widehat E_{\pi(u)-hp-1}&=0,\\
		\widehat E_{\pi(u)-hp+1}
		&=\begin{cases}
			f^{t-h}_u,&u\in\mathcal V,\\
			0,&u=(\mathrm{in},j).
		\end{cases}
	\end{aligned}
\end{equation}
By the uniqueness of coefficient indices established above,
\eqref{eq:st-program}--\eqref{eq:st-companions} and
\eqref{eq:st-read-positions} give, for every $v\in\mathcal V$,
\begin{align}
	\sum_{\ell=1}^{p+1}a^{(0)}_\ell
	\widehat E_{\pi(v)-\ell}
	&=\sum_{h,u}c_{vhu}f_u^{t-h}
	=g^t_v-b_v,\label{eq:st-aux-base}\\
	\sum_{\ell=1}^{p+1}a^{(0)}_\ell
	\widehat E_{\pi(v)+1-\ell}
	&=\sum_{h,u\in\mathcal V}
	c_{vhu}(f_u^{t-h}-f_u^{t-h})=0.
	\label{eq:st-aux-copy}
\end{align}
The same index uniqueness and the absence of nonzero
block entries outside the positions in \eqref{eq:st-encode}
give
\begin{equation}\label{eq:st-aux-zeros}
	\begin{aligned}
		\sum_{\ell=1}^{p+1}a^{(0)}_\ell
		\widehat E_{\pi(\mathrm{in},j)-\ell}&=0,\\
		\sum_{\ell=1}^{p+1}a^{(0)}_\ell
		\widehat E_{-\ell}&=0.
	\end{aligned}
\end{equation}

For $r\in\{0,\pi(\mathrm{in},j),\pi(v),\pi(v)+1\}$,
where $v\in\mathcal V$ and $0\le j<m$, the block encoding gives
\begin{equation}\label{eq:st-control-inputs}
	\begin{aligned}
		\widehat E_{r-\pi(v)-e}
		&=\ind{r=\pi(v)+e}
		&& (e\in\{0,1\}),\\
		\widehat E_{r-\ell}&=0
		&& (p/3<\ell\le2p/3),\\
		\widehat E_{r-1}
		&=\begin{cases}
			f^t_v,&r=\pi(v)+1,\\
			0,&\text{otherwise},
		\end{cases}\\
		\widehat E_{r-p}-\widehat E_{r-p+1}
		-\widehat E_{r-p-1}
		&=\begin{cases}
			1,&r=0,\\
			f^{t-1}_{\mathrm{in},j},
			&r=\pi(\mathrm{in},j),\\
			0,&r=\pi(v)\text{ or }\pi(v)+1.
		\end{cases}
	\end{aligned}
\end{equation}
Substituting \eqref{eq:st-aux-base}--\eqref{eq:st-control-inputs}
and the coefficients in \eqref{eq:st-control}
into \eqref{eq:st-score} gives
\begin{equation}\label{eq:st-base-scores}
\begin{aligned}
 s_{\pi(v)}&=g^t_v,\\
 s_{\pi(v)+1}&=G(2f^t_v-1),\\
 s_{\pi(\mathrm{in},j)}&=(B+G)f^{t-1}_{\mathrm{in},j}-B,\\
 s_0&=G.
\end{aligned}
\end{equation}
Let
\[
\Pi=\{0\}\cup
\{\pi(\mathrm{in},j):0\le j<m\}
\cup\bigcup_{v\in\mathcal V}\{\pi(v),\pi(v)+1\}.
\]
We prove that $s_r<-G$ for every
$r\in\{0,\ldots,p-1\}\setminus\Pi$.

By \eqref{eq:st-control-size} and \eqref{eq:st-control},
\[
\sum_{\ell=1}^{p+1}\max\{a_\ell,0\}
\le 15B+3G+\sum_{\ell=1}^{p+1}|a^{(0)}_\ell|
+\sum_{v\in\mathcal V}|b_v|<B^2.
\]
If $p/3<r\le2p/3$, then $a_r=-B^2$ and
$\widehat E_0=1$, so \eqref{eq:st-score} gives
\[
s_r\le-B-B^2+
\sum_{\ell=1}^{p+1}\max\{a_\ell,0\}<-B<-G.
\]

Now suppose $r\notin\Pi$ and
$r\notin(p/3,2p/3]$.
By \eqref{eq:st-encode}, every nonzero position of
$\widehat E$ belongs to $(\Pi-p)\cup\Pi$.
Since $\pi(v)\in(p/6,p/3-1)$ for every
$v\in\mathcal V$, it follows that
$$
\widehat E_{r-p}=0,~~
\widehat E_{r-\pi(v)}
=\widehat E_{r-\pi(v)-1}=0 ~~~(v\in\mathcal V).
$$
Hence, the terms with coefficients $a_p$,
$a_{\pi(v)}$, and $a_{\pi(v)+1}$ vanish.
The coefficients $a_{p-1}$, $a_{p+1}$, and
those in $(p/3,2p/3]$ are nonpositive.
Since $a_1=2G$ and all bits belong to $\{0,1\}$,
\eqref{eq:st-score} gives
\[
s_r\le-B+\sum_{\ell=1}^{p+1}|a^{(0)}_\ell|
+2G<-G,
\]
where the last inequality follows from
\eqref{eq:st-control-size}. Thus,
\begin{equation}\label{eq:st-other-zero}
 s_r\le-B+\sum_{\ell=1}^{p+1}|a^{(0)}_\ell|+2G<-G.
\end{equation}

Set $w=a$ and $b=-B$ in \eqref{eq:stochastic-model}.
For every $t\in[K]$ and $0\le r<p$, let $U_{t,r}$
be independent uniform random variables on $[0,1]$.
Define the ideal samples by
$f^t_{\test}=\ind{U_{t,\pi(\test)}\le p_x}$,
and the remaining ideal bits by \eqref{eq:st-gates}--\eqref{eq:st-address-copy}.
Initialize $(Z_{-p},\ldots,Z_0)=(0,E^0_x)$.
For $n=(t-1)p+1+r$, generate
\[
Z_n=\ind{U_{t,r}\le
	\sigma\left(-B+\sum_{\ell=1}^{p+1}a_\ell Z_{n-\ell}\right)},
\]
and define $\widehat E^t_r=Z_n$.

Let $A_{t,r}$ be the event that $\widehat E^s_j=E^s_j$
at every position $(s,j)$ preceding $(t,r)$ in lexicographic order.
On $A_{t,r}$, the score of $Z_n$ equals $s_r$ in
\eqref{eq:st-score}.
If $r=\pi(\test)$, then $s_r=g^t_{\test}$, so
$\widehat E^t_r=E^t_r$ by the shared variable $U_{t,r}$.
At every other position,
\eqref{eq:st-base-scores} and \eqref{eq:st-other-zero} give
$
(2E^t_r-1)s_r\ge G/2.
$
Since $U_{t,r}$ is independent of all preceding variables,
$
\Pr(\widehat E^t_r\ne E^t_r\mid A_{t,r})
\le e^{-G/2}.
$
The events
$A_{t,r}\cap\{\widehat E^t_r\ne E^t_r\}$
are disjoint and cover every possible first disagreement.
Hence,
\[
\Pr(\exists t\in[K]:\widehat E^t\ne E^t)
\le pKe^{-G/2}.
\]
Finally, $E^K_{\pi(\out,2)}=Y_K$ by
\eqref{eq:st-state-invariant} and \eqref{eq:st-encode},
so the claimed probability bound follows.

\end{proof}

\subsection{Choosing the Stochastic Inputs}\label{sec:stochastic-inputs}

For each $x\in\{0,1\}^m$, choose
\begin{equation}\label{eq:st-prompts}
 \widehat x=(0,E^0_x)\in\{0,1\}^{p+1}.
\end{equation}
The prompts are distinct and independent of the labels by
\eqref{eq:st-init} and \eqref{eq:st-encode}.
The first generated bit has block position $0$, so the output
of round $K$ occurs at step
\begin{equation}\label{eq:st-horizon}
 T_m=(K-1)p+1+\pi(\out,2)<Kp.
\end{equation}

\begin{proof}[Proof of Theorem~\ref{thm:stochastic}]
Fix $m\ge2$ and put $N=2^m$. For arbitrary labels
Given arbitrary labels $\lambda_0,\ldots,\lambda_{N-1}\in\{0,1\}$, choose $\alpha$ by \eqref{eq:st-code}
and apply Lemma~\ref{lem:st-simulation} with
\begin{equation}\label{eq:st-choose-G}
 G=2\log(2pK/\delta_N)\ge1.
\end{equation}
By \eqref{eq:st-state-invariant} and \eqref{eq:st-horizon},
the lemma's weights and bias satisfy
\begin{equation}\label{eq:st-coupling-gap}
 \left|q^{(T_m)}_{w,b}(\widehat x)-\Pr(Y_K=1)\right|
 \le pKe^{-G/2}=\delta_N/2.
\end{equation}
Lemma~\ref{lem:st-reader} gives, for every $x$,
\begin{equation}\label{eq:st-shattering}
 (2\lambda_{j(x)}-1)
 \left(q^{(T_m)}_{w,b}(\widehat x)-\tfrac12\right)
 \ge\delta_N/2>0.
\end{equation}
Thus $\mathcal H_{p+1,T_m}$ shatters the $2^m$ prompts.
For sufficiently large $d$, choose the largest $m\ge2$ with
\begin{equation}\label{eq:st-choose-m}
 p+1=6144(m+1)+1\le d.
\end{equation}
Prepend zeros to the prompts and set additional weights to zero;
the output probabilities are unchanged. Since $m=\Theta(d)$,
\begin{align}
 \operatorname{VCdim}(\mathcal H_{d,T_m})&\ge2^m=2^{\Omega(d)},
 \label{eq:st-final-vc}\\
 T_m&<6144(m+1)\,4^{2\cdot2^m}\le2^{2^{O(d)}}.
 \label{eq:st-final-horizon}
\end{align}
The horizon depends only on $d$.
\end{proof}

	\section*{Acknowledgments}
	Idan Mehalel is supported by the European Research Council (ERC) under the European Union's Horizon 2022 research and innovation program (grant agreement No. 101041711), the Israel Science Foundation (grant number 2258/19), and the Simons Foundation (as part of the Collaboration on the Mathematical and Scientific Foundations of Deep Learning).
	Ilan Doron-Arad is supported by grant NSF DMS-2031883 and Vannevar Bush Faculty Fellowship ONR-N00014-20-1-2826 (PI Mossel).
	Elchanan Mossel is partially supported by NSF DMS-2031883, Vannevar Bush Faculty Fellowship ONR-N00014-20-1-2826, MURI N000142412742, and a Simons Investigator Award.
	
	\paragraph{AI assistance.}
	GPT Astra found the main proof technique. Iterative work with
	GPT Astra simplified the construction and its presentation. The
	authors take full responsibility for the paper.

\end{document}